\documentclass{l4dc2025}
\usepackage{times}  % DO NOT CHANGE THIS
\usepackage{wrapfig}
\usepackage{caption} % 

\usepackage{stackengine}
\usepackage{mathrsfs}
\usepackage{mathtools}
\usepackage{booktabs}
\usepackage{amsfonts} 
\usepackage{cases}
\usepackage{float}
\usepackage{multirow}
\usepackage{color}

\usepackage{bm}

\newcommand{\dataset}{\mathcal{D}}

\newcommand{\outset}{\mathbb{R}^n}

\DeclareMathOperator*{\argmin}{arg\,min}

\newcommand{\refhere}{\textbf{\textcolor{red}{[Reference here]}}}

\newcommand{\Rb}{\ensuremath \mathbb{R}}
\newcommand{\indim}{\ensuremath m}
\newcommand{\outdim}{\ensuremath n}

\newcommand{\paramdim}{\ensuremath p}
\newcommand{\paramSet}{\ensuremath W}

\newcommand{\Prob}[1]{
    \ifthenelse{\equal{#1}{X}}{
        \ensuremath \mu
    }{
        \ifthenelse{\equal{#1}{Y}}{
            \ensuremath \nu
        }{
            \ensuremath \lambda
        }
    }
}

\newcommand{\empProb}[2]{
    \ifthenelse{\equal{#1}{X}}{
        \ensuremath \hat \mu_{#2}
    }{
        \ifthenelse{\equal{#1}{Y}}{
            \ensuremath \hat \nu_{#2}
        }{
            \ensuremath \hat \lambda_{#2}
        }
    }
}

\newcommand{\dirac}[1]{\ensuremath \delta_{#1}}

\newcommand{\minimize}{\ensuremath \mathrm{minimize}}
\newcommand{\subjecto}{\ensuremath \mathrm{subject~to}}

\newcommand{\perf}{\phi^{\mathtt{Perf}}}
\newcommand{\tool}{\textsc{RAC-BNN}\xspace}
\newcommand{\xy}[1]{{\color{magenta} {#1}}}

\newcommand{\cvar}{\text{CVaR}}

\title[Risk-Averse Certification of BNNs]{Risk-Averse Certification of Bayesian Neural Networks}
\author{
\Name{Xiyue Zhang}$^1$\thanks{Current address: School of Computer Science, University of Bristol, UK (xiyue.zhang@bristol.ac.uk)}\Email{xiyue.zhang@cs.ox.ac.uk} \\
\Name{Zifan Wang}$^2$\Email{zifanw@kth.se} \\
\Name{Yulong Gao}$^3$ \Email{yulong.gao@imperial.ac.uk} \\
\Name{Licio Romao}$^4$ \Email{licro@dtu.dk} \\
\Name{Alessandro Abate}$^1$ \Email{alessandro.abate@cs.ox.ac.uk} \\
\Name{Marta Kwiatkowska}$^1$ \Email{marta.kwiatkowska@cs.ox.ac.uk} \\
\addr $^1$  Department of Computer Science, University of Oxford, UK\\
\addr $^2$  Division of Decision and Control Systems, KTH Royal Institute of Technology, Sweden\\
\addr $^3$ Department of Electrical and Electronic Engineering, Imperial College London, UK\\
\addr $^4$  Department of Wind and Energy Systems, Technical University of Denmark, Denmark
}

\begin{document}

\maketitle

\begin{abstract}
In light of the inherently complex and dynamic nature of real-world environments, incorporating risk measures is crucial for the robustness evaluation of deep learning models.
In this work, we propose a \textbf{R}isk-\textbf{A}verse \textbf{C}ertification framework for Bayesian neural networks called \tool.
% \mk{the tool name should reflect 'risk averse', as otherwise it is too general} 
% \textit{BNN-Cert}.
Our method leverages sampling and optimisation to compute a sound approximation of the output set of a BNN, represented using a set of template polytopes. 
To enhance robustness evaluation, we integrate a coherent distortion risk measure—Conditional Value at Risk (CVaR)—into the certification framework, providing probabilistic guarantees based on empirical distributions obtained through sampling.
We validate \tool on a range of regression and classification benchmarks and compare its performance with a state-of-the-art method.
The results show that \tool effectively quantifies robustness under worst-performing risky scenarios, and achieves tighter certified bounds and higher efficiency in complex tasks.
\end{abstract}
\section{Introduction}

% \textbf{Outline of the introduction.}
% NNs
There has been  growing interest in formal verification of neural networks~~\citep{huang2017safety,katz2017reluplex,zhang2018crown,singh2019deeppoly,tjeng2019evaluating,xu2020automated}, in particular when  deploying deep neural models to  safety- and security-critical systems, autonomous vehicles~\citep{Bojarski16,Codevilla18}, healthcare systems~\citep{Alipanahi15}, and cyber security~\citep{DahlSDY13,ShinSM15}. Different from deterministic neural networks which learn a fixed set of weights and biases from a set of training data, Bayesian neural networks (BNNs) provide a principled approach to modelling uncertainty \citep{neal2012bayesian} and  learn a posterior distribution over these network parameters. 
During inference, BNNs quantify uncertainty and assign high uncertainty values to out-of-distribution inputs instead of being overconfident in wrong predictions~\citep{kahn2017uncertainty}. 
%Meanwhile, 
At the same time, 
the stochastic nature of BNNs complicates certification, as both the model parameters and, as a result, the predictive outputs are probability distributions rather than point estimates. 
Even when applying relaxation-based certification techniques to BNNs, the computational complexity can increase drastically.

In order to reliably deploy BNN solutions and reason about their safety in the presence of uncertainty, techniques have been developed to handle stochastic constraints within BNNs and compute certified bounds on their reachable outputs. These approaches typically fall into two categories: sampling-based techniques, which provides probabilistic guarantees~\citep{cardelli2019statistical, wicker2020probabilistic,michelmore2020uncertainty}, and approximation-based techniques, which evaluate the robustness of BNNs by computing the expectation of the output distributions over an input set \citep{Adams23, wicker2024adversarial}.
While approximation techniques significantly increase the scalability of evaluating larger-size BNNs, they often introduce relaxation losses to the certified output range.
Such relaxation can lead to conservative output bounds, limiting the precision of robustness evaluations.
Furthermore, existing works focus on bounding the expectation of the entire output distribution.
However, in real-world decision-making scenarios, considering the average performance over the full distribution may not suffice. Instead, it is important to account for challenging scenarios by adopting a \textit{risk-averse} perspective; 
that is, to evaluate robustness under adverse conditions (e.g., the most adversarially unstable 25\% cases).

% such as the worst-performing subsets of the distribution,
% prone to be unstable in their predictions with even minor perturbations on the input. 
% This is quite problematic for high-stakes applications.
% Therefore, Bayesian neural networks are proposed to principally model uncertainty in the weight parameters.

% Our contributions and how they address the gap
In this work, we highlight the importance of a risk-averse perspective for BNN certification.
Specifically, we propose a principled approach to BNN certification that incorporates coherent distortion risk measures -- Conditional Value at Risk (CVaR) \citep{rockafellar2000optimization} -- which enables flexible and targeted evaluation of BNN performance.
The key idea of our method is to sample the input points and the parameters of the BNN weights, obtaining the empirical output distribution, to compute a sound approximation of the output set (using template polytopes) and certified CVaR bounds with  probabilistic guarantees. 
We implement our method as a prototype tool, \tool, and demonstrate that it achieves tighter certification bounds  with better efficiency than state-of-the-art techniques on a range of regression and classification benchmarks.
To the best of our knowledge, \tool is the only method capable of computing certified bounds under different risk levels (denoted by $\alpha$), 
enabling the flexibility between analysing average robustness over the entire output distribution ($\alpha=1$) and evaluating robustness against worst-performing outcomes ($\alpha<1$).

% This paper makes the following contributions:
% In contexts where attention needs to be paid to risky environments or the most dangerous adversarial interventions, 
% How we validate the results

\section{Related Work}
We now discuss closely related works in the certification of BNNs
and risk-averse learning.

% \mk{This does not cover decision robustness, see contribution and compare to related work of \url{http://fun2model.org/papers/wpl+24.pdf } (note that UAI 2020 only defines probabilistic robustness, which is not used here)}
\paragraph*{\emph{Robustness Certification of BNNs}} The last decade has witnessed a growing interest in formal certification of neural networks, including complete verification methods based on constraint solving \citep{huang2017safety,katz2017reluplex,tjeng2019evaluating} and incomplete verifiers based on convex relaxation \citep{zhang2018crown,singh2019deeppoly,xu2020automated}. However, these methods all assume deterministic neural networks with fixed weights and thus cannot be directly applied to certify BNNs.
To this end, a series of certification techniques have been proposed for the certification of BNNs \citep{cardelli2019statistical, wicker2020probabilistic,Adams23,wicker2024adversarial}.

\cite{cardelli2019statistical} proposed a statistical approach to estimate the probability of the existence of adversarial examples with \textit{a priori} guarantees by viewing the robustness of a BNN as a Bernoulli random variable.  
% This approach was built on a sequential scheme \citep{Jegourel18} and iteratively sampled the BNN posterior to check whether the induced deterministic neural network (with the sampled weights) is robust or not.
\cite{wicker2020probabilistic} focused on the probabilistic robustness of BNNs, that is, the probability of the sampled weights from the posterior for which the resulting deterministic neural network satisfying a safety property. The method computes a certified lower bound for the probabilistic safety based on relaxation techniques of interval and linear bound propagation. 
These bounds are later generalised in
\cite{pmlr-v161-wicker21a}, where they are applied to bound sequential decisions on BNN-based models, and specifically in \cite{WICKER2024104132} for reach-avoid (bounded-until) specifications. 
\cite{wicker2024adversarial} further investigated decision robustness, which focuses on the decision step aligned with Bayesian decision theory, as also used in this work, and proposed a unified approach to compute certified lower and upper bounds for both probabilistic robustness and decision robustness.
% \cite{Berrada21} generalised the Lagrangian duality and proposed a functional Lagrangian framework for verifying BNNs on probabilistic specifications. 
\cite{Adams23} leveraged dynamic programming to bound the output range of BNNs over an input region. 
There are also certification methods that are able to reason about the robustness of the closed-loop systems where BNNs are applied for decision-making. 
\cite{michelmore2020uncertainty} introduced a statistical framework to evaluate the safety of end-to-end BNN controllers in autonomous driving.
% \cite{lechner2021infinite} focused on computing safe weight sets in BNN policies and ensuring the safety of the closed-loop systems by recalibrating BNN policies via rejection sampling.

\paragraph*{\emph{Risk-averse learning}}
Risk-averse learning has emerged as a critical area in machine learning, particularly for applications where decisions have significant consequences under uncertainty. Traditional machine learning models often focus on minimising expected loss, which may not adequately capture the potential for rare but severe adverse outcomes. To solve this, researchers have investigated risk-averse approaches that consider not just the expected performance but also the tail risks \cite{vitt2019risk,lakdawalla2021health,o2018modeling,tamar2015policy}. 
For example, \cite{vitt2019risk} introduced a risk-averse classification framework leveraging coherent risk measures to address class-specific misclassification risks, demonstrating its effectiveness through applications to support vector machines.
In the field of healthcare engineering, \cite{lakdawalla2021health} introduced the risk-adjusted cost-effectiveness framework, which integrates risk aversion and diminishing returns into health technology assessments. By accounting for tail risks and variability in treatment outcomes, this approach addresses limitations of traditional cost-effectiveness analysis, particularly for severe illnesses and uncertain interventions.
% In the field of reinforcement learning, \cite{tamar2015policy} proposes a policy gradient method that optimizes policies with respect to coherent risk measures, allowing for control over the trade-off between expected returns and risk.
% \xyc{zifan: add more risk-averse related literature or evaluation under risk measures?}

\section{Preliminaries and Problem Formulation}
\label{sec:problem-statement}

Next we introduce the necessary background and notations to be employed throughout the paper.

\paragraph*{\emph{Notation}}
We denote the input space by $\mathcal{X} \subseteq \Rb^\indim$, the output space by $\mathcal{Y} \subseteq \Rb^\outdim$, and the parameter space by $\mathcal{\paramSet} \subseteq \Rb^\paramdim$.
% We use $\mathcal{X} \subseteq \Rb^\indim$ and $\mathcal{Y} \subseteq \Rb^\outdim$ to denote the input and output space, respectively, and  $\mathcal{\paramSet} \subseteq \Rb^\paramdim$  the parameter space. 
We denote by $\mathcal{P}(\mathcal{X})$ the set of probability distribution over $\mathcal{X}$, that is,
$\mathcal{P}(\mathcal{X}) = \left\{ \Prob{X}: \int_\mathcal{X} \Prob{X}(d\xi) = 1, \Prob{X} \geq 0 \right\}$,
and similarly for $\mathcal{Y}$ and $\mathcal{W}$. 
% Throughout the paper, the Euclidean spaces are equipped with their corresponding Borel $\sigma$-algebra and all measures are defined with respect to it. More details can be found in \refhere. 
For a finite collection of points $\{x_1, \ldots, x_N\}$ in $\mathcal{X}$, we denote the corresponding empirical distribution as
$\empProb{X}{N}(x) = \frac{1}{N} \sum_{i = 1}^N \dirac{x_i} (x)$, 
where $\dirac{x_i} (x)$ denotes the Dirac measure centered at $x_i$.
Similarly, for points in the output space $\mathcal{Y}$ and parameter space $\mathcal{\paramSet}$, we denote the empirical distribution by $\empProb{Y}{N}$ and $\empProb{\paramSet}{N}$, respectively. Given two probability distributions $\Prob{X}$ and $\Prob{X}'$ defined on the input space, i.e., $\Prob{X}, \Prob{X}' \in \mathcal{P}(\mathcal{X})$, we denote by $W_1(\Prob{X},\Prob{X}')$ the \textit{type-1 Wasserstein distance} between these measures, defined as 
% \mk{integrate wrt $d\pi$?}
\begin{equation}
    W_1(\Prob{X},\Prob{X}') = \inf_{\pi \in \Pi(\Prob{X},\Prob{X}')} \int_{\mathcal{X}\times \mathcal{X}} \| \xi_1 - \xi_2 \| d\pi( \xi_1,  \xi_2),
    \label{eq:Kant-metric}
\end{equation}
where $\Pi(\Prob{X}, \Prob{X}')$ is the set of couplings (or joint distributions) with marginals given by $\Prob{X}$ and $\Prob{X}'$.

Conditional Value at Risk (CVaR) is a coherent risk measure.
For a random variable $X$ with the cumulative distribution function (CDF) denoted by $F$ and a specified risk level $\alpha \in (0,1]$, the CVaR value is given by ${\text{CVaR}}_{\alpha}[X] = \mathbb{E}_F[X| X>  {\text{VaR}}_{\alpha}[X] ]$, where ${\text{VaR}}_{\alpha}[X] = {\text{inf}}\{ y: F_X(y) \geq 1-\alpha\}$ represents the $1-\alpha$ quantile of the distribution, also known as the Value at Risk (VaR). 
Intuitively, CVaR captures the average of the worst-case outcomes within the upper $\alpha \%$ of the distribution.

\paragraph*{\emph{Bayesian Neural Networks}} In this section, we define Bayesian neural networks (BNNs) and review related concepts using the notation introduced in the previous section.

\begin{definition}[Bayesian Neural Network]\label{defi:BNN} 
Given a distribution $\Prob{W}\in \mathcal{P}(\mathcal{W})$ over the parameter space $\mathcal{W}$, a Bayesian Neural Network (BNN) is defined as a continuous stochastic function $f:\mathcal{X} \times \mathcal{W}\mapsto \mathcal{Y}$, where the weight $w$ is sampled from the distribution $\Prob{W}$, i.e., $w \sim \Prob{W}$.
\end{definition}

In the training of BNNs, we start with a prior distribution  $p(w)$ over the parameters $w$ and then compute the posterior distribution $p(w|\dataset)$  conditioned on dataset $\dataset= \{(x_i, y_i)\in \mathcal{X}\times \mathcal{Y}: i = 1, \ldots, N\}$.  Note that  the measure $\Prob{W} \in \mathcal{P}(\mathcal{W})$ in Definition~\ref{defi:BNN}  refers to  the  posterior distribution $p(w | \dataset)$.
With dataset $\dataset$ observed, the prior distribution of a BNN is updated according to the likelihood, $p(\dataset | w) = \prod_{i=1}^{n_\dataset} p(y_i | x_i, w)$, which models how likely the outputs are observed under the stochasticity of model parameters and the inputs.
The posterior distribution, given the dataset, is then computed by virtue of the Bayes formula, i.e., $p(w | \dataset) \propto p(\dataset | w) p(w)$. 
In practice, the posterior distribution $p(w | \dataset)$ can be obtained by  different 
inference techniques, e.g., Hamiltonian Monte Carlo (HMC) \citep{neal2012bayesian}, Variational Inference (VI) \citep{blundell2015weight}, and Monte Carlo Dropout (MCD) \citep{gal2016dropout}.

The posterior $p(w | \dataset)$ then induces the distribution over outputs called the posterior predictive distribution for an input point $x^*$, which is defined as $   p(y^* | x^*, \mathcal{D}) = \int p(y^{*} | x^{*}, w)p(w | \mathcal{D}) dw$. 
% \begin{equation}\label{eq:pred_distr}
%    p(y^* | x^*, \mathcal{D}) = \int p(y^{*} | x^{*}, w)p(w | \mathcal{D}) dw.
% \end{equation}
The final decision is obtained using Bayesian decision theory for regression and classification, which selects the value $\hat{y}$ that minimises the corresponding loss function $\mathcal{L}$ averaged over the predictive distribution: $\hat{y} = \argmin_{y} \int_{\outset} \mathcal{L}(y, y^*) p(y^* | x^*, \dataset) dy^*$.

\begin{wrapfigure}{r}{0.4\textwidth}
    \centering
    \vspace{-0.5cm}
\includegraphics[width=0.4\columnwidth]{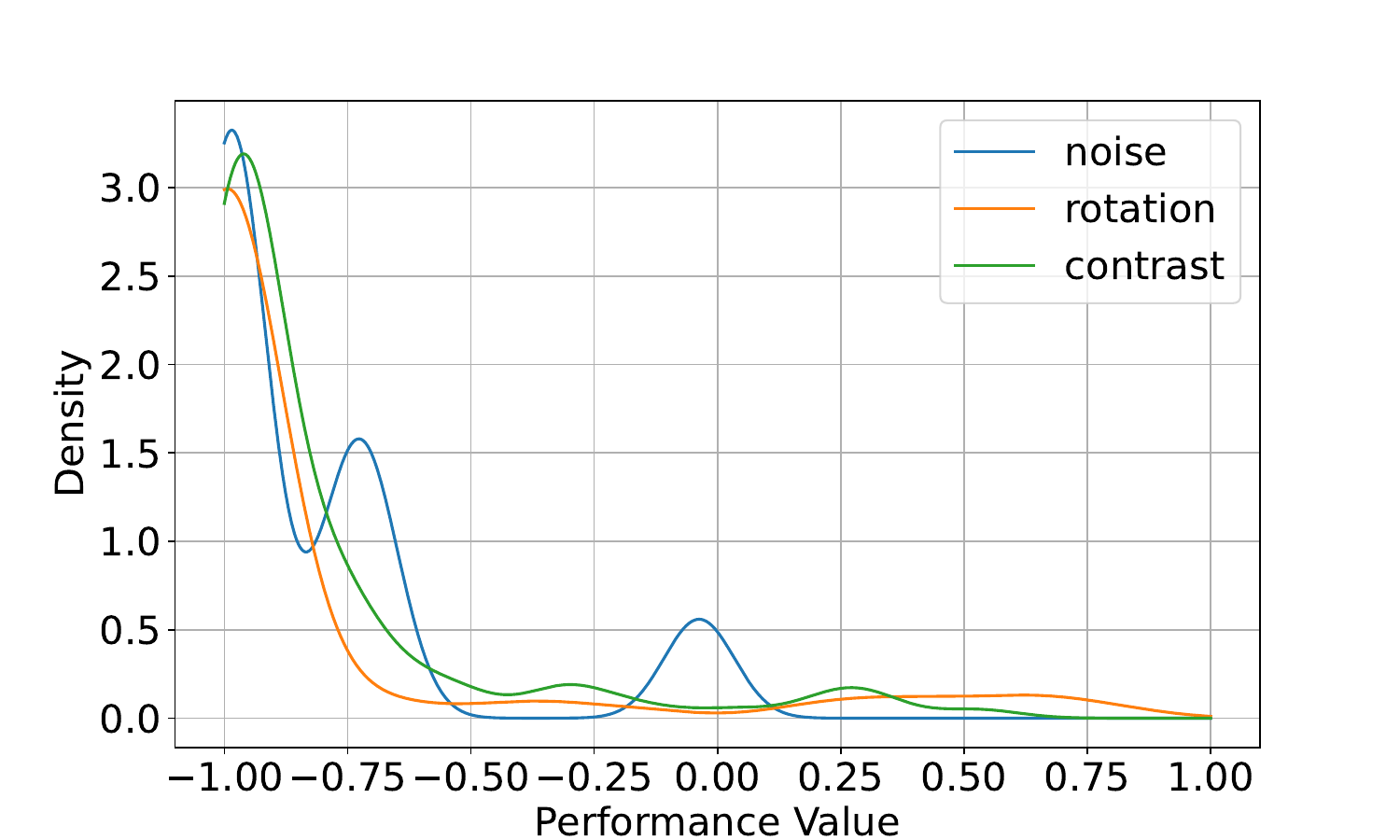} 
\caption{Tail distributions exist in Bayesian neural networks when recognising images with different types of perturbations.}
\label{fig:tail_perf_dist_pert}
\end{wrapfigure}
\paragraph*{\emph{Motivating Example}}
In this section, we present a motivating example to demonstrate why investigating the expectation of the entire distribution alone is insufficient to evaluate the performance of BNNs. In high-stakes applications, outputs that deviate significantly from the safe region and lead to catastrophic consequences are unacceptable, even if their probability of occurrence is low. In such cases, relying solely on the expectation of the outputs fails to account for the risks, as illustrated in the following example. 

We consider the MNIST classification task and visualise the empirical output distribution of a BNN under three types of perturbations: Gaussian noise applied to all image pixels, rotation that alters image orientation, and changes in brightness contrast.
To evaluate the robustness performance, we define the function $h(y)=\max_{t \in [10]\setminus c} y_t-y_c$, where $y_t$ denotes the random variable of the BNN output for the second-largest probability class, and
$y_c$ denotes the BNN output for the ground-truth class $c$. 
The value of $y$ ranges from -1 to 1, with -1 indicating the BNN is robust and correctly classifies the input, while 1 indicates the BNN makes a confident incorrect decision.
A value $h(y) > 0$ indicates that the BNN is not robust to the perturbations. 

As shown in Figure~\ref{fig:tail_perf_dist_pert}, we observe tail distribution bumps under Gaussian noise within the interval around $[-0.1, 0.1]$, contrast perturbations within the interval around $[0.1, 0.3]$ and under rotation perturbations lying around $[0.25, 0.75]$.
These tails correspond to cases where the BNN confidently makes incorrect predictions. In risk-sensitive or safety-critical applications, such tail risks can lead to severe consequences.
Standard expected values fail to capture the reliability of a BNN in these high-risk scenarios, highlighting the need for risk-averse certification techniques.
\subsection{Problem Statement} 
Consider a BNN $f$, with the input $x \sim \Prob{X}$ and the parameters $w \sim \Prob{W}$, and the output given by $y = f(x,w)$.
Let $\nu$ denote the distribution of $y$. We assume that the output set $\mathcal{Y}$ is compact. This assumption holds for many tasks of BNNs. One can further render this assumption true by mapping the output to a prescribed compact set. 
For a given risk-averse level $\alpha$ and an evaluation function $h:\mathcal{Y} \rightarrow \mathbb{R}$, we define the risk-averse evaluation as  $\perf= \text{CVaR}_{\alpha,y \sim \Prob{Y}}[h(y)]$.
% \begin{equation}\label{eq:loss1}
% \perf= \text{CVaR}_{\alpha,y \sim \Prob{Y}}[h(y)].
% \end{equation} 
%
In this work, we evaluate a BNN through the following problems. 
\begin{enumerate}
\setlength\itemsep{-0.2em}
    \item 
    Output support set computation: approximate  $\mathcal{Y}=\{f(x,w) \mid x\in\mathcal{X},w\in\mathcal{W}\}$;
    \item Risk-averse evaluation:   compute certified bounds on the CVaR value of the BNN output.
\end{enumerate}

%\mk{Explain/exemplify what these would be for the motivating example}

In the motivating example, evaluating the performance of the BNN involves addressing two interrelated tasks: approximating the output set of $y_t$ and $y_c$, and computing the CVaR value of the performance function $h$. Approximating the output set enables a visual assessment of the BNN's outputs, allowing us to determine whether they fall within a safe region and meet desirable criteria. However, the set approximation alone cannot capture the probabilistic information about the likelihood of outputs lying in the safe region. The second task, i.e., computing the CVaR value, addresses %complements 
this drawback, as it quantifies the tail risks by focusing on the most extreme and potentially hazardous outcomes. Together, these tasks provide a comprehensive framework for assessing the reliability and robustness of the BNN in high-stakes applications, where both the nature of the outputs and their risk profiles are crucial considerations.

\section{Methodology}

\subsection{Output Set Approximation}
In this section, we present our sampling-based solution to approximate the output support set. 
According to the posterior distribution $\Prob{W}$ and the input distribution $\Prob{X}$, we first collect a group of i.i.d. sampled inputs $x_i\in \mathcal{X}$ and a group of i.i.d. sampled parameters $w_j\in \mathcal{W}$. 
Based on these collected samples, we can then compute the corresponding output samples  $y_{ij}=f(x_i,w_j)$.
We rewrite the output samples as $y_k$ for notation simplicity and use $N$ to denote the total number of samples.

For the output set approximation, we aim to compute a convex approximation of $\mathcal{Y}$.
Leveraging the output samples, we build the approximation by taking the intersection of all half-spaces that contain $\{y_k\}_{k=1}^{N}$.
This convex hull for the output samples and the tractable over-approximations can be conveniently represented as a template polytope, which provides high-confidence guarantees for the approximation gap by applying the scenario optimisation theory to our problem.  

Consider a convex template polytope  
$\mathbb{V}=\{z\in \mathbb{R}^n\mid Vz\leq \bm{1}\}$
where $V\in\mathbb{R}^{L\times n} $ and $L$ is the number of half spaces or inequalities.  
Given the set $\mathbb{V}$, and $\bm{ \theta}\in \mathbb{R}^{L}$, we introduce a parameterised set in the form of $\mathcal{H}(\bm{ \theta}):=	\{z\in \mathbb{R}^{n}\mid  V z\leq \bm{ \theta}\}.$
% \begin{align}
% 	\mathcal{H}(\bm{ \theta}):=	\{z\in \mathbb{R}^{n}\mid  V z\leq \bm{ \theta}\}.
% \end{align} 
%
We now approximate the output set $\mathcal{Y}$ by computing the optimal parameterised set $\mathcal{H}(\bm{ \theta}_N^\star)$ with respect to the output samples, where  $\bm{ \theta}_N^\star$ is the optimal solution to following optimisation problem 
\begin{eqnarray}\label{Opt:quanset}
	\begin{cases}
		&\min\limits_{\bm{\theta} \in \mathbb{R}^L}\quad  \bm{1}^T  \bm{\theta}    \\
		& \hspace{0.1cm}{\text s.t} \quad \quad
		
		Vy_k \leq \bm{ \theta} , k=1,\cdots,N.
		
	\end{cases}
\end{eqnarray}  

% \begin{eqnarray}\label{Opt:quanset}
% 	\begin{cases}
% 		&\min\limits_{\bm{\theta}}\quad  \bm{1}^T  \bm{\theta}    \\
% 		& \hspace{0.1cm}{\text s.t}
% 		\begin{cases}
% 			Vy_k \leq \bm{ \theta} , k=1,\cdots,N, \\
% 			\bm{ \theta} \in \mathbb{R}^L.
% 		\end{cases}
% 	\end{cases}
% \end{eqnarray}  

The optimisation result is presented in the following proposition. 
\begin{proposition}\label{Prop: risk bound}
	The optimal solution $\bm{ \theta}_N^\star$ to the optimisation problem in Equation \eqref{Opt:quanset} is  
\begin{eqnarray}  
	[\bm{ \theta}_N^\star]_i=\max_{k=1,\cdots,N} [V]_i y_k,
\end{eqnarray}  
	where $[V]_i$ denotes the $i$-th row of $V$.  
Let $\hat{\mathcal{Y}}_N=\mathcal{H}(\bm{ \theta}_N^\star)$.  
Given $\epsilon_1\in (0,1)$,  $\beta_1\in (0,1)$, and the Euler's constant $\text{e}$, if $N\geq \frac{1}{\epsilon_1} \frac{\text{e}}{\text{e}-1} \Bigl(\ln\frac{1}{\beta_1} + n+L \Bigr)$,
% \begin{eqnarray} \label{Eq:samplenum} 
% 	N\geq \frac{1}{\epsilon_1} \frac{\text{e}}{\text{e}-1} \Bigl(\ln\frac{1}{\beta_1} + n+L \Bigr),
% \end{eqnarray}  
then, with probability no less than $1-\beta_1$, $\mathbb{P}[y\in \mathcal{Y}: y\notin \hat{\mathcal{Y}}_N] =	\int_ {\mathcal{Y} \setminus \hat{\mathcal{Y}}_N}\Prob{Y}(d \xi)
 \leq  \epsilon_1$.
% \begin{eqnarray}  \label{Eq:errorbound}
% 	\mathbb{P}[y\in \mathcal{Y}: y\notin \hat{\mathcal{Y}}_N] =	\int_ {\mathcal{Y} \setminus \hat{\mathcal{Y}}_N}\Prob{Y}(d \xi)
%  \leq  \epsilon_1.
% \end{eqnarray}  
%   \mk{define ${\hat{\mathcal{Y}}_N^c}$}
\end{proposition}
 \begin{proof}
    Given the optimisation problem in Equation \eqref{Opt:quanset}, it is easy to see that the constraint $Vy_k \leq \bm{ \theta}$, $k=1,\cdots,N$, should be active element-wise with the optimal solution $\bm{ \theta}_N^\star$. That is, $[\bm{ \theta}_N^\star]_i=\max_{k=1,\cdots,N} [V]_i y_k$. Furthermore, the robust counterpart of the optimisation problem \eqref{Opt:quanset} is \begin{eqnarray}\label{Opt:robustquanset}
	\begin{cases}
		&\min\limits_{\bm{\theta} \in \mathbb{R}^L}\quad  \bm{1}^T  \bm{\theta}    \\
		& \hspace{0.1cm}{\text s.t} \quad \quad
		
		Vy \leq \bm{ \theta} , \forall y\in \mathcal{Y},
		
 	\end{cases}  
\end{eqnarray}  where Equation \eqref{Opt:quanset} is its corresponding scenario LP, where  $\{y_k,k=1,\cdots,N\}$ is the set of i.i.d. samples from the  set $\mathcal{Y}$. Since we assume that the set $\mathcal{Y}$ is compact, \cite[Theorem 4]{AlamoACC} yields the stated sample complexity and the confidence guarantee in Proposition~\ref{Prop: risk bound}. 
 \end{proof}

Proposition~\ref{Prop: risk bound}  provides a statistical bound on the discrepancy between the estimated output set $\hat{\mathcal{Y}}_N$ and the true output set $\mathcal{Y}$. The error bound $\epsilon_1$ has an inverse relationship with $N$, indicating that achieving tighter error bounds necessitates a significantly larger sample size.
The dimensionality of the sample space, $n$, and of the structured polytope, $L$, contributes linearly to the sample size, reflecting the increased effort needed to address higher-dimensional or structurally complex problems. Besides, the term $\ln\frac{1}{\beta_1}$ introduces a logarithmic dependence on the confidence level. These results indicate that tighter precision and higher confidence come at the cost of increased computational and data collection demands.

\subsection{Risk-Averse Evaluation}\label{sec:risk-averse-eval}
Given the output samples $\{y_k\}^N_{k=1}$, we define $\empProb{Y}{N}(y) = \frac{1}{N} \sum_{i = 1}^N \dirac{y_i} (y)$ as the empirical distribution of $\nu$.
Our goal is to estimate the value of $\cvar_{\alpha, y \sim \Prob{Y}}[h(y)]$, which represents the CVaR value of the performance function $h$ at level $\alpha$ under the distribution $\Prob{Y}$. 
% Specifically, we aim to construct a certified bound for this value with probabilistic confidence guarantees, which is presented in the following proposition. 
Specifically, we aim to construct a certified bound for this value with probabilistic confidence guarantees. 
We present the following lemmas that are useful in constructing confidence bounds.
\begin{lemma}\label{lemma:evadis} 
 \cite{boskos2023high}
Let $\empProb{Y}{N}(y) = \frac{1}{N} \sum_{i = 1}^N \dirac{y_i} (y)$ be the empirical distribution of $\nu$, where $y_i \in \mathbb{R}^n$.
Given $\beta\in (0,1)$, we have $\mathbb{P}( W_1(\Prob{Y},\empProb{Y}{N})\geq \epsilon_2) \leq \beta$,
 \begin{align}
    \mathbb{P}( W_1(\Prob{Y},\empProb{Y}{N})\geq \epsilon_2) \leq \beta
 \end{align}
 where $\epsilon_2 = \rho(\mathcal{Y})(C^{*} N^{-\frac{1}{n}} + \sqrt{n}(2 \ln \beta^{-1})^{\frac{1}{2}} N^{-\frac{1}{2}})$, $C^* = \sqrt{n} 2^{(n-2)/(2)}\Big( \frac{1}{1-2^{1-n/2}} + 2\Big)$, $\rho(\mathcal{Y})$ is the diameter of the support of $y$.  %, $p$ is the Wasserstein type parameter. 
 \end{lemma}

\begin{lemma}\label{lemma:cva-wasserstein} \cite{wang2024risk}
 Suppose $h(y)$ is $L_0$-Lipschitz in $y$. For any two probability distributions $\mathcal{D}_1$ and $\mathcal{D}_2$, we have $|{\cvar_{\alpha, y\sim \mathcal{D}_1}}  [h(y)]- \cvar_{\alpha, y\sim \mathcal{D}_2}  [h(y)]\vert \leq \frac{L_0}{\alpha} W_1(\mathcal{D}_1,\mathcal{D}_2).$
 \end{lemma}

 Based on Lemmas~\ref{lemma:evadis} and \ref{lemma:cva-wasserstein}, we construct the certified bound for BNN evaluation, which is presented in the following proposition. 
The proof follows from the above two lemmas and is omitted.

\begin{proposition}\label{prop:CVaR:Confi_bound}
Suppose that $h(y)$ is $L_0$-Lipschitz continuous in $y$. Then, we have $|\cvar_{\alpha, y \sim \empProb{Y}{N} }[h(y)] - \cvar_{\alpha, y \sim \Prob{Y}}[h(y)]| 
    \leq \frac{L_0}{\alpha} \epsilon_2(\beta)$,
% \begin{align}
%     &|\cvar_{\alpha, y \sim \empProb{Y}{N} }[h(y)] - \cvar_{\alpha, y \sim \Prob{Y}}[h(y)]| 
%     \leq \frac{L_0}{\alpha} \epsilon_2(\beta)
% \end{align}
with probability at least $1-\beta$, 
where $\epsilon_2(\beta) = \rho(\mathcal{Y})(C^{*} N^{-\frac{1}{n}} + \sqrt{n}(2 \ln \beta^{-1})^{\frac{1}{2}} N^{-\frac{1}{2}})$, $C^* = \sqrt{n} 2^{(n-2)/2}\Big( \frac{1}{1-2^{1-n/2}} + 2\Big)$, $\rho(\mathcal{Y})$ is the diameter of the support of $y$.
\end{proposition}

Proposition~\ref{prop:CVaR:Confi_bound} states that the distance between the empirical CVaR value and the true CVaR value is related to the sample size $N$.
Given a target certification range $H$ with probability at least $1-\beta$, the required number of samples is given by $\Big(  \frac{L_0\rho(\mathcal{Y}) ( C^{*} +\sqrt{n}(2 \ln \beta^{-1})^{\frac{1}{2}} ) }{\alpha H} \Big)^n$, where $n$ denotes the output dimension of the performance function. 
The sample size $N$ is inversely proportional to $H^n$, indicating that as the certified bound tightness $H$ decreases, $N$ increases exponentially with an exponent related to $n$. 
Additionally, a smaller $\alpha$, which focuses on the largest $\alpha \%$ of the distribution, necessitates a larger number of samples to accurately estimate the expected value within this specific portion of the distribution. Finally, $N$ depends logarithmically on the confidence level $\beta$, meaning that achieving higher confidence (smaller $\beta$) necessitates an increase in the sample size.

\section{Experiments}
% \xy{Evaluation questions to confirm: 1. do we put motivation examples earlier or in the evaluation section}
In this section, we present the experiment setup and evaluation results of the proposed approach. 

%\subsection{Benchmark and Setup}
\textbf{Benchmark and Baseline.} Following recent work \citep{Adams23}, we evaluate our method on three regression tasks, including a 1D noisy sine dataset where the BNN is trained on samples from 1D sine function with additive noise, a 2D equivalent of Noisy Sine, and the Kin8nm dataset where the BNN is trained on a dataset of state-space readings for the dynamics of an eight-link robot arm.
We further investigate the performance of our method on classification tasks where BNNs are trained on the MNIST and Fashion-MNIST datasets.
For each experiment, we evaluate the risk-averse robustness of BNN models against different noise and attack settings by computing the certified CVaR values under a range of risk levels.
Specifically, for risk level 1, we consider the state-of-the-art method BNN-DP in \cite{Adams23} for robustness certification of BNNs and conduct performance comparison in terms of the tightness of the certified bounds and the computation overhead.
All experiments are conducted on a cluster with Intel Xeon Gold 6252
2.1GHz CPU, and NVIDIA 2080Ti GPU.

\noindent\textbf{Evaluation metric.} 
To evaluate the certification performance of our approach,
we use $\gamma$-robustness \citep{Adams23}, which computes the difference between the upper and lower bounds on the expectation outptus of BNNs with regard to a property formulated by the performance function.
A smaller $\gamma$-robustness value implies a tighter certified bound computation.

% Visualization of Empirical Output Distribution.
 \vspace{-0.5cm}
\subsection{Evaluation Results}
% \mk{This section needs work to highlight the novel contributions. Visualisation is essential. Currently the text is too long and in part repetitive, and the main points are lost in the discussion. There needs to be a compelling result/results, also discussed for what difference this makse for the motivating example.}
% \mk{If possible, streamline the presentation by defining $h$, $H$, and other relevant parameters jointly at the beginning once, rather than repeat. }
% We evaluate our approach for risk-averse certification on the benchmarks to answer the following \mk{no need to repeat in a short paper} research questions:
% (1)

% (2) Is our approach effective in characterising extreme errors, i.e., risk-averse certified bounds, for a target safety property in different perturbation/attack settings?
% (3) What is the certification performance of our approach in tightness and efficiency, compared with the state-of-the-art?
% (4) What is the sample complexity of our approach?

\subsubsection{RQ1: Is our approach effective in characterising the output set?}

\begin{wrapfigure}{r}{0.4\textwidth}
    \centering
    \vspace{-0.5cm}
\includegraphics[width=0.4\columnwidth]{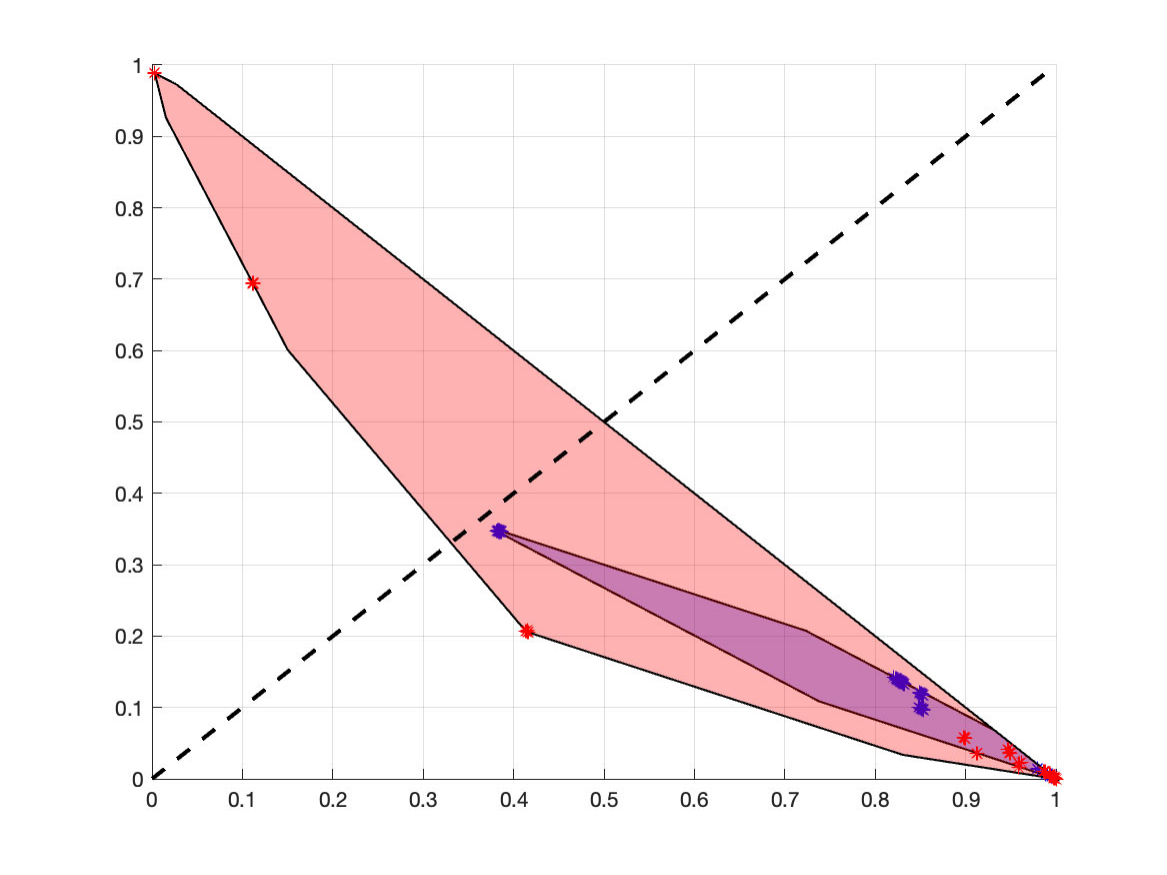} 
\caption{Output set computation for Bayesian neural networks when recognising images with different types of perturbations.}
\label{fig:tail_perf_dist}
\end{wrapfigure}
Consider again the specific case of the MNIST dataset under two distinct perturbation scenarios: rotation perturbation and noise perturbation. Figure~\ref{fig:tail_perf_dist} illustrates the relationship between two output variables of the BNN: $y_c$ representing the output probability for the true class $c$, and $y_t$ representing the output probability for the second-most likely class. 
The black dashed line $y_t = y_c$ serves as a threshold: if a data point lies below this line, the BNN classifies correctly, as the true class $y_c$ has a higher probability than $y_t$. 
Points marked in red and blue correspond to samples affected by rotation and noise perturbations, respectively. 
Given a desired error $0.05$ with a high confidence level of $95\%$, and using Proposition~\ref{Prop: risk bound}, we select $L=16$ and compute the required number of samples to be 665.

To systematically assess the BNN's robustness under such conditions, we propose computing the output support set, defined as the convex hull of all possible output pairs $(y_t,y_c)$ under each perturbation type. 
The output support set provides a visual and quantitative measure of BNN performance.
If the support set lies predominantly below the black line, the BNN demonstrates robustness against the perturbation.
Conversely, if a large portion of the support set extends above the black line, it indicates vulnerability.
Obviously, from Figure.~\ref{fig:tail_perf_dist}, the BNN shows robustness to noise perturbations but exhibits vulnerability to rotation perturbations.
This geometric perspective enables a straightforward evaluation of the BNN's performance.

\subsubsection{RQ2: Is our approach effective in characterising risk-averse performance of BNNs?}
To answer this question, we evaluate the effectiveness of our approach in characterising certified CVaR bounds under a range of risk levels. 
A comparison of certification performance for the entire output distribution ($\alpha=1$) with the baseline method \textit{BNN-DP} is deferred to Section \ref{subsec:cert_comparison}.

\begin{table}[!ht]
    \centering
  \footnotesize	
    \caption{Certified CVaR bounds under different confidence and risk levels for regression tasks.}
    \begin{tabular}{ccccccc}
        \toprule
        \multirow{2}{*}{\textbf{Tasks}} & \multicolumn{3}{c}{\textbf{$\beta=0.05$}} & \multicolumn{3}{c}{\textbf{$\beta=0.01$}} \\
        \cmidrule(lr){2-4} \cmidrule(lr){5-7}
         & \textbf{$\alpha = 1$} & \textbf{$\alpha = 0.5$} & \textbf{$\alpha = 0.25$} & \textbf{$\alpha = 1$} & \textbf{$\alpha = 0.5$} & \textbf{$\alpha = 0.25$} \\
        % \multirow{2}{*}{\textbf{Dataset}} & \multicolumn{2}{c}{\textbf{Confidence Error 1}} & \multicolumn{2}{c}{\textbf{Confidence Error 2}} & \multicolumn{2}{c}{\textbf{Confidence Error 3}} \\
        % \cmidrule(lr){2-3} \cmidrule(lr){4-5} \cmidrule(lr){6-7}
        %  & \textbf{$\alpha = 1$} & \textbf{$\alpha = 0.5$} & \textbf{$\alpha = 1$} & \textbf{$\alpha = 0.5$} & \textbf{$\alpha = 1$} & \textbf{$\alpha = 0.5$} \\
        \midrule
        1D Noisy Sine & 0.074 {\scriptsize$\pm$ 0.1} & 0.146 {\scriptsize$\pm$ 0.1}& 0.147 {\scriptsize$\pm$ 0.1} & 0.058 {\scriptsize$\pm$ 0.1} & 0.148 {\scriptsize$\pm$ 0.1} & 0.165 {\scriptsize$\pm$ 0.1} \\
        2D Noisy Sine & 0.081 {\scriptsize$\pm$ 0.1} & 0.173 {\scriptsize$\pm$ 0.1} & 0.268 {\scriptsize$\pm$ 0.1} & 0.077 {\scriptsize$\pm$ 0.1} & 0.203 {\scriptsize$\pm$ 0.1} & 0.272 {\scriptsize$\pm$ 0.1} \\
        Kin8nm & 0.080 {\scriptsize$\pm$ 0.1} & 0.088 {\scriptsize$\pm$ 0.1} & 0.097 {\scriptsize$\pm$ 0.1} & 0.079 {\scriptsize$\pm$ 0.1} & 0.088 {\scriptsize$\pm$ 0.1} & 0.103 {\scriptsize$\pm$ 0.1}\\
        \bottomrule
    \end{tabular}
    \label{tab:risk-averse-regression}
\end{table}

\noindent\textbf{Regression benchmarks.}
We first evaluate the proposed method on BNNs trained with three regression datasets: 1D Noisy Sine, its 2D equivalent, and Kin8nm.
To assess the effectiveness of our method, we compute certified CVaR bounds under different levels of risk ($\alpha=0.25, 0.5, 1$) and confidence guarantees ($\beta=0.05, 0.1$).
Table \ref{tab:risk-averse-regression} summarises the evaluation results for the regression tasks.
For each configuration, the CVaR value represents the estimated expectation of the subset of the distribution of the performance function. The maximum deviation from the true value is constrained by a pre-defined limit of $H=0.1$.
% for different risk levels  with two confidence error settings .
% With risk level $\alpha$ varying from 0.25, 0.5, to 1, we examine the expected average over the worst 25\% outcomes, the worst 50\% outcomes and the whole output distribution.
For the 1D Noisy Sine task, we apply noise up to $0.01$ around the input point $\pi/2$ and the property of interest is the deviation from the ground truth output, with the performance function defined as $h(y)=1-y$.
As expected, when we focus on the worst-case outcomes, the CVaR values, which indicate the average deviation in the subset scenarios, show an increase.
For the 2D Noisy Sine task, we use the same perturbation noise (up to 0.01), whereas the perturbations are applied to two input features rather than one. The same performance function, $h(y)=1-y$, is used for evaluation.
In this case, CVaR values show a slight increases across all settings compared to the 1D dataset, largely due to the added noises to the 2D input space.

For the Kin8nm dataset, we 
simulate input perturbation noise up to 0.01 for eight input features and evaluate the output deviation from the ground truth, formulated using the performance function $h(y)=|y^{\ast} - y|$ where $y^{\ast}$ indicates the ground-truth value.
Notably, the estimated CVaR values demonstrate small variance across difference risk levels ($\alpha=1, 0.5, 0.25$).
As previously discussed, a consistent performance under varying risk levels is the ideal risk-averse robustness we aim to achieve for BNN certification.

\begin{table*}[ht]
\centering
\footnotesize	
\caption{Certified CVaR bounds for different attacks on classification tasks. 
% \zf{Should we use a small H since the value would be lower bounded by -1? or maybe we can simply remove the range $H$.}
}
\begin{tabular}{ccccc}
\toprule
\textbf{Tasks} & \textbf{CVaR Level} & \textbf{$L_{\infty}$ noise} & \textbf{Rotation} & \textbf{Contrast} \\
\midrule
\multirow{3}{*}{MNIST} & $\alpha = 1$   & -0.999 {\scriptsize$\pm$ 0.1} & -0.998 {\scriptsize$\pm$ 0.1} & -0.997 {\scriptsize$\pm$ 0.1} \\
& $\alpha = 0.5$ & -0.999 {\scriptsize$\pm$ 0.1} & -0.998 {\scriptsize$\pm$ 0.1} & -0.993 {\scriptsize$\pm$ 0.1} \\
& $\alpha = 0.25$ & -0.999 {\scriptsize$\pm$ 0.1} & -0.997 {\scriptsize$\pm$ 0.1} & -0.986 {\scriptsize$\pm$ 0.1} \\
\midrule
\multirow{3}{*}{FASHION} & $\alpha = 1$   & -0.191 {\scriptsize$\pm$ 0.1} & -0.119 {\scriptsize$\pm$ 0.1} & -0.152 {\scriptsize$\pm$ 0.1}\\
& $\alpha = 0.5$ & 0.443 {\scriptsize$\pm$ 0.1} & 0.525 {\scriptsize$\pm$ 0.1} & 0.562 {\scriptsize$\pm$ 0.1} \\
& $\alpha = 0.25$ & 0.885 {\scriptsize$\pm$ 0.1} & 0.964 {\scriptsize$\pm$ 0.1} & 0.849 {\scriptsize$\pm$ 0.1}\\
\bottomrule
\end{tabular}
\label{tab:cvar_values}
\end{table*}
\noindent\textbf{Classification Benchmarks.}
Table 2 summarises the certified bounds for classification tasks under different levels of risk $\alpha$.
We evaluate the risk-averse robustness of BNNs against three types of attacks: $L_{\infty}$ attack where 
perturbation noise up to a specific limit is applied to all image pixels, and two geometric attacks, rotation (altering the image's orientation) and contrast (changing its brightness and contrast). 
For each configuration, the CVaR value represents the estimated expected value of the performance function, w.r.t. a property of interest, over (the subset of) the distribution.

For the MNIST dataset, the perturbation limit for the $L_{\infty}$ attack is set to 0.1,  the rotation range to $[-45^{\circ}, 45^{\circ}]$,  and the contrast factor to 0.5.
Robustness is evaluated by checking whether the predicted labels remain consistent with the ground truth label, denoted as $c^*$ among the 10 classes. The corresponding performance function is defined as $h(y)=\max_{i \in [10] \setminus {c^*} }y_{i} - y_{c^*}$, where $c^*$ is the ground-truth label. 
The BNN is robust to perturbations if $h(y)<0$.
For both $L_{\infty}$ and \textit{rotation} attacks, our approach reveals that the trained BNN achieves strong certified risk-averse robustness, with very low variance in performance across the entire output distribution and the worst 25\% outcomes.
Under the \textit{contrast} attack, the BNN shows small variance when evaluating the overall expectation and the worse-performing 25\% subset.
Nonetheless, the expected values for the most challenging cases remain far below 0, demonstrating the BNN's risk-averse robustness across all attack types. 
% robust ($0.616>0$) even though the overall expectation is less than 0.

For the FASHION dataset, we set the perturbation limit for the $L_{\infty}$ attack to 0.1, the rotation range to $[-15^{\circ}, 15^{\circ}]$, and the contrast factor to 0.9.
As with the MNIST dataset, We investigate output label consistency, formulated by the performance function $h(y)=\max_{i \in [10] \setminus {c^*} }y_{i} - y_{c^*}$. 
% The BNN is robust to the perturbations if $h(y)<0$.
Note that the expectation values over the full output distribution for all three attack types are negative, indicating overall certified robustness.
However, when focusing on the worst-performing 50\% and 25\% subset of the output distribution, the expectation values are all positive, highlighting the necessity of risk-averse evaluation for BNNs.
Compared with the BNN for the MNIST dataset, this BNN demonstrates a lack of robustness, especially in scenarios requiring a risk-averse approach.

% \mk{The points made in the text and Tables are not clear - highlight their importance and differences to SOTA}

\begin{table}[h]
\caption{Comparison with SOTA in tightness of the certified bounds and computation time.}
    \label{tab:comparison}
    \centering
    \footnotesize
    \begin{tabular}{ccccccc}
        \toprule
        \multirow{2}{*}{\textbf{Method}}  & \multicolumn{2}{c}{1D Noisy Sine} & \multicolumn{2}{c}{Kin8nm} & \multicolumn{2}{c}{MNIST}\\
        \cmidrule(lr){2-3} \cmidrule(lr){4-5} \cmidrule(lr){6-7} 
        & $\gamma$-robustness & Time (s)  & $\gamma$-robustness & Time (s)  &$\gamma$-robustness & Time (s)  \\
        \midrule
        BNN-DP & \textbf{0.065} & 6.148 & 0.137 & 7.886 & 1.572 & 24.956\\
        \tool (0.1)  & 0.2 & \textbf{0.815} & 0.2 & \textbf{0.984} & 0.2 & \textbf{1.064}\\
        \tool (0.05) & 0.1 & 1.640 & \textbf{0.1} & 1.893 & \textbf{0.1} & 4.099\\
        \bottomrule
    \end{tabular}
    
\end{table}
\subsubsection{RQ3: What is the certification performance of our approach in tightness and efficiency?}\label{subsec:cert_comparison}

We perform comparison experiments with \textit{BNN-DP} in computing certified bounds for the entire output distribution (corresponding to $\alpha=1$).
The evaluation focuses on two metrics: the tightness of the certified bounds ($\gamma$-robustness) and the computation time of the certification procedure. 
Table~\ref{tab:comparison} summarised the results for both regression and classification tasks.
In the evaluation, we present the performance results of our method (\tool) under two certification range settings, $H=0.1$ and $H=0.05$, and the $\gamma$-robustness is $2H$.
% (see Lemma \ref{lemma:cva-wasserstein}).

The evaluation results demonstrate that our method improves both the tightness and efficiency of computing certified bounds.
For the $1D$ Noisy Sine dataset, while \textit{BNN-DP} demonstrates competitive performance in certification tightness,
to achieve comparable tightness with our method requires further tightening of the certification range $H$.
%Meanwhile, 
On the other hand, 
our method significantly reduces the computational overhead for certification. 
As the input dimension increases and the task gets more complex, the advantages of our method in certification performance become more significant than the baseline method.
Specifically, when $H=0.05$, our method surpasses the baseline in tightness for all remaining tasks.
Particularly, for the MNIST dataset, our method improves the tightness of the certified bounds by 87.3\% and 93.6\% for $H=0.1$ and $H=0.05$, respectively.
In terms of computation time, our method achieves a reduction of 76.0\% and 83.6\% for Kin8nm and MNIST under the $H=0.05$ setting.

\begin{table*}[ht]
\centering
\footnotesize	
\caption{Sample complexity for different hyper-parameters.}
\begin{tabular}{cccccccc}
\toprule
 \multirow{2}{*}{\textbf{Task}} & \multirow{2}{*}{\textbf{$\gamma$-robustness}} & \multicolumn{3}{c} {\textbf{$\beta=0.05$}} & \multicolumn{3}{c}{\textbf{$\beta=0.01$}} \\
         \cmidrule(lr){3-5} \cmidrule(lr){6-8}
& & \textbf{$\alpha = 1$} & \textbf{$\alpha = 0.5$} & \textbf{$\alpha = 0.25$} & \textbf{$\alpha = 1$} & \textbf{$\alpha = 0.5$} & \textbf{$\alpha = 0.25$} \\
\midrule
\multirow{2}{*}{Kin8nm} & $H = 0.1$   & 17648 & 69923 & 278354 & 26950 & 107129 & 427177\\
& $H = 0.05$  & 69923 & 278354 & 1110733 &107129 & 427177 & 1706025\\
\midrule
\multirow{2}{*}{MNIST} & $H = 0.1$   & 2520 & 9835 &  38844 & 3808 & 14986 & 59445\\
& $H = 0.05$ & 9835 & 38844 & 154379 & 14986 & 59445 & 236783 \\
\bottomrule
\end{tabular}
\label{tab:sample_complextiy}
\end{table*}

\subsubsection{RQ4: What is the sample complexity of our approach?} 
In Section \ref{sec:risk-averse-eval}, Proposition~\ref{prop:CVaR:Confi_bound} provides a theoretical relation between the sampling complexity and key hyper-parameters, including precision tightness, risk levels and confidence errors.
For the benchmark tasks, the performance function $h(y)$ -- which measures the output difference to the ground-truth value for regression tasks or label consistency for classification tasks -- is a scalar random variable. With the dimension $n$ reduced to 1, according to Proposition~\ref{prop:CVaR:Confi_bound}, the required sample size $N$ is inversely proportional to $H$ and the risk level $\alpha$, while depending logarithmically on the confidence level $\beta$.

To demonstrate the practical sampling feasibility of our approach, we compute the required sample size under different configurations to evaluate the impact of individual parameters on  sampling complexity.
The results are summarised in Table \ref{tab:sample_complextiy}.
Across all configurations, the sampling complexity is manageable, with the largest required sample size reaching $10^6$.
This is the case when evaluating a narrow 25\% worst-performing portion of the distribution with a tightness bound of 0.05 and a 99\% confidence guarantee.

% \xyc{analysis for the relationship between sample complexity and parameters}

\section{Conclusion}
We introduce a novel risk-averse evaluation method for computing certified bounds of Bayesian neural networks using the coherent risk measure CVaR. 
By leveraging sampling and optimisation, our approach approximates the output distribution and bounds the CVaR values with probabilistic guarantees.
We implement this method in a tool, \tool, and demonstrate its ability to compute sound approximations of output sets and quantify robustness under worst-performing conditions. The results show that \tool achieves improved certification tightness and better efficiency compared to existing methods.  
An interesting future direction is to extend our approach to the safety evaluation of closed-loop dynamical systems with BNN controllers.
%the prior method.
% \mk{Future work, e.g. improve sample complexity as in Cardelli et al 2019?}

\newpage
\acks{
MK and XZ received partial support from ELSA: European Lighthouse on
Secure and Safe AI project (Grant No. 101070617 under UK guarantee) and
the ERC under the European Union’s
Horizon 2020 research and innovation program (FUN2MODEL, Grant No. 834115).}
\bibliography{bibfile}
% \newpage
% \input{Section/Appendix}
\end{document}